\documentclass[sigconf, nonacm]{acmart}

\newcommand\vldbavailabilityurl{https://github.com/minghaochen/selspec-csm}
\newcommand\vldbpagestyle{plain}

\usepackage{booktabs}
\usepackage{amsmath,amsthm}
\usepackage{algorithm}
\usepackage{algpseudocode}
\usepackage{tikz}
\usetikzlibrary{positioning,arrows.meta,fit,backgrounds}

\newtheorem{theorem}{Theorem}
\newtheorem{lemma}{Lemma}

\newtheorem{observation}{Observation}

\newcommand{\selspec}{\textsc{SelSpec}}
\newcommand{\gate}{\textsc{AnchorGate}}
\newcommand{\ball}[2]{B_{#1}(#2)}
\newcommand{\lap}{L}
\newcommand{\eig}[1]{\lambda_{#1}}

\begin{document}
\title{Can Aggregate Invariants Accelerate Continuous Subgraph Matching? Limits, Laws, and a Dynamic Spectral Index}

\author{Minghao Chen}
\affiliation{%
  \institution{Tencent Technology}
  \city{Shenzhen}
  \country{China}
}
\email{monychen@tencent.com}

\author{Jiale Zheng}
\affiliation{%
  \institution{HUAWEI Noah's Ark Lab}
  \city{Shenzhen}
  \country{China}
}
\email{zhengjiale2@huawei.com}

\begin{abstract}
Spectral filtering recently delivered substantial pruning for
\emph{static} subgraph matching: Laplacian interlacing rejects
candidates whose neighborhoods cannot host the query. We study whether
such aggregate structural tests can accelerate \emph{continuous}
subgraph matching (CSM) over dynamic graphs, and answer in three parts.
First, lazily maintained spectral bounds are infeasible exactly where
spectral pruning has value: we characterize the tightest safe rule over
a formalized perturbation relaxation and show that even it loses
essentially all pruning power within four touching updates.
Second, exact maintenance is affordable when selective: pruning utility
and recomputation cost are anti-correlated across vertices---hubs
provably never prune---so recomputing small-neighborhood spectra on
touch sustains exact local spectra at microseconds per update, complete
by construction.
Third, integrated into a decoupled CSM benchmark against an
identical-minus-spectra control, the tests remove up to $51\%$ of
candidates or safely skip up to $47\%$ of update enumerations, yet
enumeration intermediates remain unchanged---beyond the gates' skipped first-level bindings, typically zero---across two engines,
four real graphs, two stream types, and $77$ solved queries; a
constructed radius-stratified
workload confirms the instrument detects the exception when one exists
($-99.9\%$ intermediates, $748\times$ faster). Aggregate tests
accelerate what scales with candidate sets---construction, list
scans---never adjacency-guided exploration. We distill an
intermediate-invariance methodology for evaluating CSM filters and
release a reusable dynamic local-spectra index.
\end{abstract}

\maketitle

\pagestyle{\vldbpagestyle}

\ifdefempty{\vldbavailabilityurl}{}{
\vspace{.3cm}
\begingroup\small\noindent\raggedright\textbf{PVLDB Artifact Availability:}\\
The source code, data, and/or other artifacts have been made available at \url{\vldbavailabilityurl}.
\endgroup
}

\section{Introduction}
\label{sec:intro}

Continuous subgraph matching (CSM) monitors all embeddings of a query
graph $Q$ in a data graph $G$ that receives a stream of edge insertions
and deletions, reporting the matches created or destroyed by each
update~\cite{turboflux,symbi,calig,newsp}. Its applications---intrusion
detection, fraud-ring monitoring, real-time recommendation---share two
characteristics: updates arrive at high rates, and queries of interest
are large and structurally complex. Published CSM indexes filter with
labels, degrees, and neighborhood label multisets, refined by
consistency propagation over candidate adjacency (edge views, dynamic
candidate spaces)~\cite{csmsurvey,sun2022indepth,csmexperiments};
recent work adds synopsis-based dominance embeddings~\cite{divine}. What none
uses is an \emph{aggregate structural invariant}---spectra, closed-walk
counts---of candidate neighborhoods; and notably, consistency
propagation operates on exactly the adjacency structure that, as we
will show, carries enumeration cost, which is part of why aggregate
tests find no purchase. When labels are uninformative or queries grow
large (PILOS reports label-based filters failing wholesale beyond
roughly sixteen query vertices~\cite{pilos}), the label-based signals
collapse and enumeration explodes.

Static subgraph matching recently found a richer signal.
PILOS~\cite{pilos} indexes, for every data vertex, the top eigenvalues
of the Laplacian of its $h$-hop neighborhood, and prunes a candidate
whenever the query neighborhood's spectrum cannot \emph{interlace} into
the candidate's: a subgraph's padded Laplacian spectrum is dominated
position-wise by its supergraph's, so the test is one-sided safe.
On low-label graphs and large queries this removes roughly a quarter of
candidates and solves query sizes at which all label-based methods time
out. Spectra encode neighborhood-global structure---connectivity mass,
expansion, cycle content---that no label or degree filter sees.

The obvious research program, and the one this paper set out to execute,
is to bring spectral filtering to CSM. Two designs suggest themselves
immediately. \emph{Lazy bounds}: a single edge update perturbs a
neighborhood Laplacian by a rank-one term of norm at most~2, so Weyl-type
inequalities maintain certified eigenvalue intervals at $O(1)$ per
update, refreshed occasionally. \emph{Exact recomputation}: neighborhood
spectra are recomputed whenever touched. Both preserve the one-sided
safety of interlacing, hence completeness. The questions are whether
either is affordable, and whether the resulting pruning helps.

This paper answers both questions completely, and the answers are not
what we expected when we began. Our study proceeded through three
experiment-driven reversals, each of which reshaped the design; we
report all three because the intermediate failures carry as much
information as the final system.

\paragraph{Reversal 1: lazy bounds are impossible where they matter.}
Pruning margins---the gaps $\sum_t \max(0,\eig{t}(Q)-\eig{t}(v))$
separating query from candidate spectra---are $O(1)$ in exactly the
sparse-neighborhood regime where spectral pruning fires, while each
touching edge injects eigenvalue mass~2. We prove an optimality result
over a formalized perturbation relaxation: the tightest safe rule
observing only the stale spectrum and an insert count is the
disjunction of a trace-budget test (via majorization) with composed
interlacing-shift ceilings; and we show empirically that even this
optimal rule retains under $60\%$ of pruning power after one touching
update and essentially none after four (\S\ref{sec:maintenance}).
The entire lazy-bound family is thereby closed off.

\paragraph{Reversal 2: exactness is affordable---selectively.}
The same sparsity that kills lazy bounds makes exact recomputation cheap:
a 32-vertex neighborhood Laplacian solve costs tens of microseconds.
Moreover, pruning utility and maintenance cost are
\emph{anti-correlated} across vertices: we prove that any vertex whose
$h$-hop ball contains a vertex of degree $\ge |V(Q)|$ can never be
pruned on its leading eigenvalue---hubs are spectrally unprunable---and
measure that vertices with balls above 64 nodes never prune at all,
while carrying the largest solve costs. A \emph{selective exact} index
(\selspec) that deploys spectra only on small-ball, label-relevant
vertices and recomputes on touch retains $96\!-\!100\%$ of pruning
benefit at $\sim\!300\mu s$ per update, completeness holding by
construction (\S\ref{sec:selspec}).

\paragraph{Reversal 3: the pruning does not transfer.}
We integrated \selspec\ into a state-of-the-art decoupled CSM evaluation
framework~\cite{csmexperiments} as a drop-in index layer, against the
strictest possible control: the identical pipeline minus the spectral
test. Candidates fall by $9\!-\!51\%$, growing with query size as in the
static setting; the number of incremental matches is preserved exactly.
Yet across three query scales and two label regimes, the count of
enumeration intermediates is \emph{exactly equal} with and without
spectral filtering: every vertex the spectral test removes is one the
adjacency-guided delta enumeration never visits. We then designed a
CSM-native mechanism unavailable to static matching---an \emph{anchored
gate} that skips a whole delta enumeration when the sub-query within
radius $\rho$ of the matched query edge cannot embed in the ball around
the updated data edge---and extended it beyond spectra to triangle
counts, which are subgraph-monotone closed-walk invariants outside the
reach of degree information. The gates fire on $5\!-\!47\%$ of updates,
never wrongly; the intermediates remain identical a fourth and a fifth
time (\S\ref{sec:filtering},~\S\ref{sec:experiments}).

These five replications, together with a spectral-graph-theoretic
explanation, yield the central result of the paper:

\begin{quote}
\textbf{Aggregate-invariant boundary regularity.} In update-driven
CSM whose pipeline already saturates label/degree filtering,
subgraph-monotone aggregate tests---Laplacian interlacing, vertex and
edge counts, triangle counts---pruned only work that adjacency-guided
enumeration discarded within its first expansion levels, in every
natural workload we measured; the expensive failures were
injective-assignment-level, invisible to every aggregate.
\end{quote}

The Laplacian case is not an accident: by the Grone--Merris--Bai
theorem, the Laplacian spectrum is majorized by the conjugate degree
sequence, so top-eigenvalue interlacing is, up to majorization slack, a
degree-sequence test---and degree mass is precisely what CSM pipelines
already filter. The regularity also delimits its own scope: aggregate
pruning pays exactly where cost scales with candidate-set
size---static pipelines' candidate-space construction and the per-level
candidate-list scans of list-driven engines, which is where PILOS's
gains live and why they cannot transfer to delta enumeration, which has
no construction term; and consumers that are not adjacency-guided
(periodic batch re-enumeration, approximate-matching admission control)
remain valid
targets for the gate machinery we built.

\paragraph{Contributions.} In summary:
\begin{itemize}
  \item \textbf{Impossibility of lazy spectral maintenance}
  (\S\ref{sec:maintenance}): the disjunction of a trace-budget test and
  composed interlacing-shift ceilings is the tightest safe rule over a
  formalized perturbation relaxation of (stale spectrum, insert count),
  and even it collapses within $O(1)$ touches in the regime where
  spectral pruning fires.
  \item \textbf{A selective exact maintenance primitive}
  (\S\ref{sec:selspec}): an anti-correlation law and a hub-exclusion
  theorem justify deploying exact spectra only where they can prune;
  the resulting index maintains exact local spectra at microsecond
  update cost, independent of graph size, and is reusable for any
  dynamic-graph application consuming local spectra.
  \item \textbf{The aggregate-invariant boundary regularity}
  (\S\ref{sec:filtering}): replicated across the nine configurations of
  Table~\ref{tab:matrix} ($87$ co-solved queries) with zero
  completeness violations, explained via Grone--Merris--Bai and an
  aggregate-versus-assignment argument; plus \gate, a safe, novel
  update-level mechanism whose CSM behavior completes the regularity
  and whose validity persists for non-adjacency-guided consumers.
  \item \textbf{A methodology} (\S\ref{sec:experiments}): the
  intermediate-invariance test, which detects when candidate reductions
  are enumeration-irrelevant; we show candidate counts mislead in both
  directions, extending the caveat raised by recent CSM benchmarking
  work~\cite{csmexperiments}.
\end{itemize}

\section{Preliminaries}
\label{sec:prelim}

\subsection{Continuous Subgraph Matching}
A data graph $G=(V,E,\ell)$ and query graph $Q=(V_Q,E_Q,\ell)$ carry
vertex labels $\ell(\cdot)$. An \emph{embedding} is an injective mapping
$f:V_Q\to V$ preserving labels and edges. Given a stream of edge
insertions and deletions applied to $G$, CSM reports after each update
the set (or count) of embeddings created or destroyed. All competitive
CSM systems follow an \emph{index--enumerate} paradigm
\cite{turboflux,symbi,calig,newsp,csmexperiments}: an index maintains,
for each query vertex $u$, a candidate set $C(u)\subseteq V$
(and often candidate \emph{edge views} for query edges); each update
triggers a delta enumeration anchored at the updated edge, exploring
only candidates adjacent to the partial embedding under construction.
We call this exploration \emph{adjacency-guided}: a candidate
participates only if reachable from the anchor through candidate edges.
The standard vertex filters are LDF (label and degree) and NLF
(per-label neighbor counts); both are subgraph-monotone necessary
conditions.

\subsection{Neighborhood Spectra and Safe Pruning}
For $S\subseteq V$ let $G[S]$ denote the induced subgraph and $\lap(S)$
its combinatorial Laplacian. Let $\ball{h}{v}$ be the set of vertices
within distance $h$ of $v$, and let
$\eig{1}(S)\ge\eig{2}(S)\ge\cdots$ be the eigenvalues of $\lap(S)$,
padded with zeros beyond $|S|$. Two classical facts give one-sided
safety. First, adding edges adds a PSD term to the Laplacian, and
padding appends zeros, so if $H$ embeds into $G[S]$ as a subgraph then
$\eig{t}(H)\le\eig{t}(S)$ for all $t$. Second, embeddings do not
increase distances, so an embedding $f$ maps the query ball
$\ball{h}{u}\cap Q$ into $G[\ball{h}{f(u)}]$.

We write $\lambda^Q_t(u) := \eig{t}\!\big(Q[\ball{h}{u}\cap V_Q]\big)$
for the query-side ball spectrum at $u$ and
$\lambda^G_t(v) := \eig{t}\!\big(G[\ball{h}{v}]\big)$ for the data-side
ball spectrum at $v$.

\begin{lemma}[Interlacing safety~\cite{pilos}]
\label{lem:safety}
If $f$ is an embedding with $f(u)=v$, then
$\lambda^Q_t(u) \le \lambda^G_t(v)$ for all $t$.
Hence pruning $v$ from $C(u)$ whenever
$\exists t:\ \lambda^Q_t(u) > \lambda^G_t(v)$
never removes a true match.
\end{lemma}

The same argument applies to any \emph{subgraph-monotone aggregate
invariant} $\phi$: if $H\subseteq G[S]$ implies $\phi(H)\le\phi(S)$,
then $\phi(Q\text{-side})>\phi(\text{data-side})$ is a safe pruning
test. Vertex and edge counts, all Laplacian eigenvalues, and closed-walk
counts $\mathrm{tr}(A^k)$ (e.g.\ $k{=}3$: six times the triangle count)
are all subgraph-monotone. This family is the object of our study.

\section{Lazy Spectral Bounds Are Infeasible}
\label{sec:maintenance}

A single edge insertion inside $\ball{h}{v}$ perturbs $\lap(\ball{h}{v})$
by $E=(e_a-e_b)(e_a-e_b)^{\!\top}$, PSD with $\|E\|_2=2$;
deletions subtract such a term. Weyl's inequalities~\cite{hornjohnson}
control the eigenvalue movement under such perturbations. The lazy program maintains certified
upper bounds $\bar\lambda$ on the data-side spectrum: deletions only
lower true eigenvalues (bounds stay valid for free), and insertions
are absorbed by cheap bound updates, with full recomputation deferred
until the bounds lose discriminating power. Pruning with upper bounds is
safe by Lemma~\ref{lem:safety}. The question is how fast the power decays.

\subsection{Three Bound Families and an Optimality Result}
\label{sec:lazyfamilies}

\paragraph{Flat Weyl.} $\bar\lambda_i \mathrel{+}= 2$ per inserted edge,
for every $i$. A subtlety with real systems: one edge update can pull a
new vertex carrying $k>1$ edges into the ball; accounting $+2$ per
\emph{update} rather than per \emph{induced edge} is unsafe (we observed
genuine violations), so the reverse index must track induced edge deltas.

\paragraph{Rank-one shift.} For one inserted edge,
$\eig{1}'\le\eig{1}+2$ but $\eig{i}'\le\eig{i-1}$ for $i\ge 2$
(Weyl with $\mu_2(E)=0$): the bound vector \emph{shifts} rather than
inflates, tightening the flat rule by ${\sim}40\%$ in our measurements.

\paragraph{Trace budget.} $k$ inserted edges add total eigenvalue mass
exactly $2k$, and each eigenvalue can only increase. For the stale
spectrum $d$ (descending) to come to dominate the query spectrum $q$,
the minimal mass needed is the majorization cost
$C(q,d)=\sum_t \max(0,\,q_t-d_t)$. Hence
\begin{equation}
\label{eq:trace}
\text{prune iff } C(q, d^{\text{stale}}) > 2k_{\text{ins}}
\end{equation}
is safe, costs $O(r)$, and dominates both per-component rules whenever
several components sit near their margins.

Neither per-component family subsumes the other, and neither does the
trace rule alone: the shift constraint can disprove domination when the
mass budget cannot (a large $q_i$ above the ceiling $d_{i-k}$), and
vice versa. To speak of optimality we must fix the constraint set.
Define the \emph{perturbation relaxation} $\mathcal{R}(d,k)$ as the set
of descending sequences $\lambda'$ satisfying the three families we can
prove for $k$ composed rank-one PSD insertions:
(M)~$\lambda'_i\ge d_i$ (monotonicity);
(S)~$\lambda'_i\le c_i :=
\min_{0\le j\le \min(k,\,i-1)}\big(d_{i-j}+2(k-j)\big)$, the lattice-path
minimum over all interleavings of the per-step alternatives ``interlace
one position'' or ``gain at most norm $2$''---mixed paths can dominate
both pure options when the spectrum decays steeply, so neither the pure
shift $d_{i-k}$ nor the pure norm bound is the minimizer in general;
(T)~$\sum_{i\le r}(\lambda'_i-d_i)\le 2k$ (trace budget).
Our measured rule composes the per-step pointwise minimum
$b'_1=b_1+2$, $b'_i=\min(b_{i-1},\,b_i+2)$ $k$ times, which is exactly
the dynamic program computing the ceilings $c_i$.
$\mathcal{R}$ ignores graph realizability and the finer interpolating
Weyl inequalities (whose perturbation eigenvalues are not observable
from $k$ alone), so it over-approximates the truly reachable spectra:
rules safe w.r.t.\ $\mathcal{R}$ are safe in reality.

\begin{theorem}[Optimal rule in the relaxation]
\label{thm:tight}
The rule ``prune iff $C(q,d^{\text{stale}}) > 2k_{\text{ins}}$
\textbf{or} $\exists i:\, q_i$ exceeds its ceiling in (S)'' is the
tightest safe rule over $\mathcal{R}(d^{\text{stale}},k_{\text{ins}})$:
it prunes exactly when no $\lambda'\in\mathcal{R}$ dominates $q$.
\end{theorem}

\begin{proof}[Proof sketch]
The pointwise-minimal candidate $\lambda'_i=\max(q_i,d_i)$ is descending
(maximum of two descending sequences), satisfies (M) by construction,
satisfies (S) iff every $q_i$ respects its ceiling, and consumes mass
exactly $C(q,d)$, the minimum any dominating sequence must spend under
(M). Hence a dominating member of $\mathcal{R}$ exists iff both
conditions hold.
\end{proof}

We state plainly what this does and does not claim: optimality is
relative to $\mathcal{R}$; a rule exploiting structure outside
$\mathcal{R}$ (e.g.\ graph realizability) could in principle be
tighter. The collapse below is measured for the optimal-in-$\mathcal{R}$
rule itself, and Observation~\ref{obs:tension} explains why any rule
whose admissible set inflates with injected trace mass meets the same
fate.

\subsection{Collapse}
\label{sec:collapse}

Theorem~\ref{thm:tight} makes the empirical question sharp: if even the
optimal-in-$\mathcal{R}$ rule fails, every rule in this information
class fails. On sparse graphs with dense queries---the regime where
exact interlacing prunes $7$--$27\%$ of candidates beyond NLF
(\S\ref{sec:experiments})---we measured retention, the fraction of
exact-spectrum prunes still made by each rule, as a function of the
number $\tau$ of touching updates since refresh
(Figure~\ref{fig:probe}b). Flat Weyl, the composed ceilings $c_i$, and
the optimal trace$\vee$ceiling rule retain $54$/$56$/$58\%$ at
$\tau{=}1$, $14$/$16$/$19\%$ at $\tau{=}2$, and nothing beyond
$\tau{=}4$.
The mechanism is structural: prune margins $C(q,d)$ concentrate between
1 and 8 in sparse neighborhoods, while each touching update injects mass
${\approx}5$ on average---margins and perturbations share a scale, so no
bound family can hold a useful invariant across even a handful of
updates. Safety was never violated (with correct edge accounting);
power simply evaporates.

\begin{observation}
\label{obs:tension}
The regimes are inseparable: spectral pruning requires sparse, small
neighborhoods (rich ones dominate every query spectrum), and sparse,
small neighborhoods are exactly where one edge is a large relative
perturbation. Lazy maintenance is infeasible precisely where it would
be useful.
\end{observation}

\section{Selective Exact Maintenance}
\label{sec:selspec}

Observation~\ref{obs:tension} has a constructive reading: in the useful
regime, neighborhoods are small, so \emph{exact} recomputation is cheap.
The design question becomes \emph{where} exactness must be paid.

\subsection{The Anti-Correlation Law and Hub Exclusion}

\begin{theorem}[Hub exclusion, top-$t$ form]
\label{thm:hub}
Let $d_t$ denote the $t$-th largest degree within $G[\ball{h}{v}]$ and
$q_{\max}=\max_Q|V_Q|$ over registered queries. If
$d_t \ge q_{\max}+t-1$, then the component-$t$ test
$\lambda^Q_t(u)>\lambda^G_t(v)$ never fires for any registered query:
the Brouwer--Haemers bound $\eig{t}(\lap)\ge d_t-t+2$, valid under its
mild edge-count precondition~\cite{brouwerhaemers} (satisfied here,
since a vertex of degree $d_t\ge t$ supplies the required edges),
gives $\lambda^G_t(v)\ge q_{\max}+1$, while every query-ball eigenvalue
is at most its vertex count $\le q_{\max}$.
\end{theorem}

In particular ($t{=}1$) a single vertex of degree $\ge q_{\max}$ inside
the ball kills the leading component; empirically, a short run of
high-degree vertices kills the top components, which carry
$88$--$100\%$ of observed prunes.

Empirically the exclusion extends far beyond component~1: across an
Erd\H{o}s--R\'enyi family, a Watts--Strogatz family, and a real
collaboration network, the per-vertex pruning frequency against a dense
query workload decays monotonically with ball size---$0.90$, $0.60$,
$0.24$, $0.065$, $0.000$ for ball sizes $1\!-\!8$, $9\!-\!16$,
$17\!-\!32$, $33\!-\!64$, $\ge\!65$ on the real graph
(Figure~\ref{fig:probe}d)---while per-solve cost grows with ball size.
Utility and cost are anti-correlated: \emph{the most expensive vertices
to maintain are precisely those that never prune}. Component attribution
shows the top two eigenvalues fire $88\!-\!100\%$ of all prunes, so
$r\in\{2,4\}$ components suffice---two to four floats per deployed
vertex.

\subsection{The \selspec\ Index}
\label{sec:selspecindex}

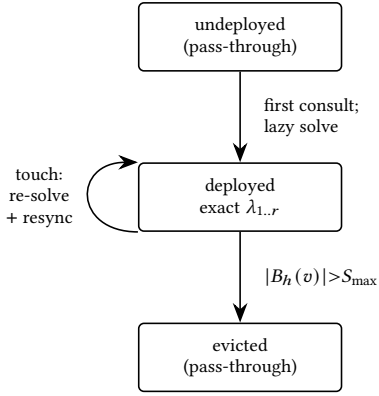
\begin{figure}
\centering
\begin{tikzpicture}[font=\footnotesize, node distance=12mm,
  st/.style={draw, semithick, rounded corners=2pt, inner sep=4pt,
             align=center, text width=24mm, minimum height=9mm},
  ar/.style={-{Stealth[length=2.6mm]}, semithick}]
\node[st] (lazy) {undeployed\\(pass-through)};
\node[st, below=of lazy] (dep) {deployed\\exact $\lambda_{1..r}$};
\node[st, below=of dep] (ev) {evicted\\(pass-through)};
\draw[ar] (lazy) -- node[right=2mm, align=left]{first consult;\\lazy solve} (dep);
\draw[ar] (dep) -- node[right=2mm]{$|\ball{h}{v}|{>}S_{\max}$} (ev);
\draw[ar] (dep.south west) to[out=180, in=180, looseness=2.4]
  node[left=1mm, align=center]{touch:\\re-solve\\$+$ resync} (dep.north west);
\end{tikzpicture}
\caption{\selspec\ vertex lifecycle. Deployment is lazy (first
consult), refresh is eager on touch (Theorem~\ref{thm:complete}'s
resync), eviction is permanent and always safe.}
\label{fig:lifecycle}
\end{figure}

\selspec\ deploys spectra on the vertex set
$\{v: |\ball{h}{v}|\le S_{\max}\}$, intersected with label-relevant
candidates of registered queries. A vertex's lifecycle
(Figure~\ref{fig:lifecycle}) is: \emph{lazy
deploy} on first consult (its spectrum is computed when a query first
tests it, so build cost scales with consulted candidates, not $|V|$);
\emph{eager refresh} on every touching update while deployed;
\emph{evict} when the ball outgrows $S_{\max}$. Spectra are maintained
\emph{exactly}:
\begin{itemize}
\item \textbf{Trigger.} An update $(a,b)$ changes $G[\ball{h}{v}]$ iff
both endpoints lie in the (new) ball; a reverse ball-membership index
locates affected deployed vertices in $O(1)$ amortized.
\item \textbf{Recompute on touch.} Affected spectra are recomputed by a
dense symmetric eigensolve on $\le S_{\max}$ vertices (tens of
microseconds via LAPACK); vertices whose balls outgrow $S_{\max}$ are
\emph{evicted}: marked pass-through and dropped from the reverse index
at its next compaction. Eviction is permanent in our implementation---a
ball that later shrinks back below $S_{\max}$ is not re-deployed, which
forfeits pruning on that vertex but never completeness (pass-through
admits everything); re-deployment would only require re-running the
lazy first-consult path.
\item \textbf{Bidirectional resync.} An insertion can \emph{raise} an
affected vertex's spectrum, re-legalizing a previously pruned candidate;
affected vertices are therefore re-checked against all label-matching
query vertices in both directions. Deletions only lower true spectra,
so stale entries over-admit---never over-prune---and completeness is
unconditional.
\end{itemize}

\begin{theorem}[Completeness]
\label{thm:complete}
Under any update sequence, \selspec's candidate sets contain every
vertex participating in any embedding. \end{theorem}
\begin{proof}[Proof sketch]
Spectra consulted at pruning time are exact for deployed vertices
(eager refresh on touch; lazily computed on first consult), undeployed
vertices pass through, and the test is safe by Lemma~\ref{lem:safety};
the resync step restores any candidate whose spectrum change reverses a
past rejection.
\end{proof}

\begin{theorem}[Update cost]
\label{thm:cost}
Let $A$ be the set of deployed vertices whose induced ball subgraph is
changed by the update (both endpoints inside the new ball). Per edge
update, maintenance costs
$O\!\big(|A|\cdot(S_{\max}^{3} + S_{\max}\,\bar d)\big)$---one dense
eigensolve plus a capped-ball BFS re-extraction with average degree
$\bar d$ per affected vertex---independent of $|V|$.
\end{theorem}

\begin{lemma}[Reverse-index maintenance]
\label{lem:rev}
The reverse ball-membership index stores, for each deployed vertex,
its $\le S_{\max}$ ball members ($O(D\cdot S_{\max})$ memory for $D$
deployed vertices); locating $A$ costs $O(|R_a|+|R_b|)$ list reads for
the endpoint lists $R_a,R_b$, which self-compact to live entries at
each access; and a recompute re-registers at most $S_{\max}$ entries,
so amortized maintenance is $O(|A|\cdot S_{\max})$ per update.
\end{lemma}

Empirically, $|A|\approx 4.5$ on a real collaboration graph with the
end-to-end kernel near $313\mu s$ per update; the same machinery
costs ${\sim}0.5$ms per update on com-Amazon where nearly all vertices
are deployable, and the index (spectra plus reverse lists) reaches
${\sim}10\times$ the NLF index footprint there
(\S\ref{sec:exp:amazon})---the memory price of deployability.

On the real graph, $S_{\max}{=}32$ deploys $63.6\%$ of vertices and
retains $96.1\%$ of all pruning; $S_{\max}{=}64$ retains $100\%$
(Figure~\ref{fig:probe}d). \selspec\ is, to our knowledge, the first
dynamic-graph index maintaining exact local Laplacian spectra, and is
application-agnostic: any consumer of neighborhood spectra over a
dynamic graph (anomaly scoring, graph kernels, spectral features)
inherits the primitive.

\section{Filtering: Where Aggregates End}
\label{sec:filtering}

\subsection{Vertex-Level Filtering and the Weak-Candidate Effect}
\label{sec:weakcand}

Plugged into a decoupled CSM framework as an index layer over NLF,
\selspec\ removes $9\!-\!51\%$ of candidate vertices, growing with
query size, at unchanged match outputs (\S\ref{sec:experiments}).
Yet enumeration time does not improve, and the count of enumeration
intermediates is \emph{identical} to the spectral-free control at every
query size. The diagnosis is structural. The interlacing test fires on
candidates whose neighborhoods are spectrally \emph{poorer} than the
query's; but the delta enumeration is adjacency-guided---a candidate
contributes cost only if a partial embedding reaches it through
candidate edges, which requires local structural support. Spectrally
poor candidates lack precisely that support and die at the first
adjacency check. We call this the \emph{weak-candidate effect}:

\begin{observation}
\label{obs:weak}
In adjacency-guided delta enumeration, the prunable set of the
interlacing test and the cost-bearing set of the enumeration are nearly
disjoint. Where the test \emph{does} pay is in cost terms that scale
with candidate-set size: static pipelines pay a large per-query
candidate-space construction term, and engines in the
GraphQL/worst-case-optimal-join style~\cite{graphql,wcoj} additionally
scan or intersect
materialized candidate lists per level; CS-edge-guided engines
(CECI/DAF style~\cite{ceci,daf}) chase adjacency even in the static setting and their
exploration is as immune as delta enumeration's
(\S\ref{sec:exp:static}).
\end{observation}

\subsection{Anchored Gates: a CSM-Native Mechanism}
\label{sec:gate}

CSM offers a handle static matching lacks: every incremental match
contains the updated edge. If $(a,b)$ is matched to query edge $(u,u')$,
then for any $\rho$, the sub-query
$Q_\rho(u,u')=Q[\{w:\min(d_Q(w,u),d_Q(w,u'))\le\rho\}]$
must embed into $G[\ball{\rho}{\{a,b\}}]$. Any subgraph-monotone
invariant $\phi$ therefore yields a safe \emph{update-level} gate:

\begin{theorem}[Gate safety]
\label{thm:gate}
If for every label-compatible query edge $(u,u')$ there exist $\rho$ and
a subgraph-monotone $\phi$ with
$\phi(Q_\rho(u,u'))>\phi(G[\ball{\rho}{\{a,b\}}])$, then the update
$(a,b)$ creates no incremental match, and its delta enumeration can be
skipped entirely.
\end{theorem}

The symmetric statement holds for deletions: every \emph{disappearing}
match contains the deleted edge, and the same ball argument applies on
the pre-deletion graph, so a failing test there licenses skipping the
negative delta enumeration as well.

\begin{figure}
\centering
\begin{tikzpicture}[font=\scriptsize, scale=0.9,
  ar/.style={-{Stealth[length=2mm]}}]
\node at (-0.2,1.62) {$Q$};
\draw[thick] (0.6,1.2) -- (1.4,1.2);
\node[fill=black, circle, inner sep=1.2pt, label=below:{$u$}] at (0.6,1.2) {};
\node[fill=black, circle, inner sep=1.2pt, label=below:{$u'$}] at (1.4,1.2) {};
\foreach \p in {(0.2,1.8),(1.0,2.0),(1.8,1.8)} \node[fill=gray, circle, inner sep=1pt] at \p {};
\draw[gray] (0.6,1.2)--(0.2,1.8) (0.6,1.2)--(1.0,2.0) (1.4,1.2)--(1.0,2.0) (1.4,1.2)--(1.8,1.8) (0.2,1.8)--(1.0,2.0);
\draw[dashed, rounded corners] (-0.15,0.9) rectangle (2.15,2.3);
\node at (1.0,0.66) {$Q_\rho(u,u')$};
\node at (3.3,1.62) {$G$};
\draw[very thick] (4.1,1.2) -- (4.9,1.2);
\node[fill=black, circle, inner sep=1.2pt, label=below:{$a$}] at (4.1,1.2) {};
\node[fill=black, circle, inner sep=1.2pt, label=below:{$b$}] at (4.9,1.2) {};
\foreach \p in {(3.7,1.8),(4.5,2.0)} \node[fill=gray, circle, inner sep=1pt] at \p {};
\draw[gray] (4.1,1.2)--(3.7,1.8) (4.5,2.0)--(4.9,1.2);
\draw[dashed, rounded corners] (3.45,0.9) rectangle (5.55,2.3);
\node at (4.5,0.66) {$\ball{\rho}{\{a,b\}}$};
\draw[ar] (2.35,1.6) -- node[above]{$\phi(Q_\rho)\!\le\!\phi(B_\rho)$?} (3.35,1.6);
\node[align=center] at (4.5,2.62) {update edge $(a,b)$};
\end{tikzpicture}
\caption{\gate: every incremental match maps $Q_\rho(u,u')$ into
$\ball{\rho}{\{a,b\}}$; a subgraph-monotone invariant $\phi$ failing
for every label-compatible query edge licenses skipping the whole
delta enumeration (Theorem~\ref{thm:gate}).}
\label{fig:gate}
\end{figure}
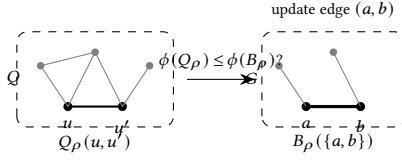
We instantiate $\phi$ with vertex count, edge count, top-$r$ Laplacian
eigenvalues, and triangle count ($\mathrm{tr}(A^3)/6$), evaluated at
$\rho\in\{1,2\}$ with a ball-size cap (capped balls are rich and pass
for free, by the anti-correlation law). Sub-query invariants are
precomputed at registration; the data ball is assembled per update in
microseconds. Crucially, $Q_\rho$ of a dense query is far richer than
any single-vertex neighborhood, so the gate's discriminating power
exceeds vertex-level tests by construction, and triangle counts add a
\emph{cycle-mass} dimension invisible to degree information.

\subsection{The Boundary Regularity}
\label{sec:law}

The gates fire on $5\!-\!47\%$ of updates with zero wrong skips---and
the enumeration intermediates remain exactly unchanged, a fourth time
for the Laplacian gate and a fifth for the triangle gate
(\S\ref{sec:experiments}). The updates that pass the gate are exactly
the expensive ones. Together with \S\ref{sec:weakcand}, five
configurations spanning two granularities (vertex, update) and two
invariant families (degree mass, cycle mass) replicate one phenomenon,
which we state as an \emph{empirical regularity} backed by a mechanism
explanation---not as a theorem:

\begin{quote}
\emph{In label/NLF-saturated, update-driven CSM, subgraph-monotone
aggregate tests pruned only work that adjacency-guided enumeration
discarded within its first expansion levels---in every configuration
we measured.}
\end{quote}

Two arguments explain the observations, with different reach.
\emph{(i) Degree mass is largely spoken for.} By Grone--Merris (proved
by Bai)~\cite{gronemerris,bai2011}, the Laplacian spectrum is majorized
by the conjugate degree sequence, which \emph{bounds the headroom} of
top-eigenvalue interlacing over degree-based filters such as NLF. The
bound is not exhaustive: spectra do carry information beyond
degrees---$C_6$ and $2{\times}C_3$ share a degree sequence yet differ
in $\lambda_1$ ($4$ vs.\ $3$)---and that cycle-structure residue is
precisely why spectral signals are attractive at all. GMB explains
attenuation, not absence---and the attenuation is quantifiable: over
the $3{,}506$ deployed balls of ca-GrQc, the top-$4$ Laplacian partial
sum reaches a median $73\%$ of its conjugate-degree-sequence ceiling
(IQR $59$--$91\%$), so most of what the interlacing statistic measures
on this workload is degree-determined, with a real but minority
cycle-structure residue. \emph{(ii) Aggregate deficiency tends to be
self-punishing.} A region poor in the tested aggregate lacks raw
structural material, and adjacency-guided search starved of material
tends to halt within its first expansion levels; expensive failures, in
everything we measured, arose where material abounds but no
\emph{injective, label-consistent assignment} exists---a property no
multiset aggregate of the region can witness. This is a mechanism, not
a guarantee. Configurations escaping it are conceivable: a region
abundant within radius~1 yet deficient at radius~2 can host
combinatorially many injective partial assignments before exhausting,
which a $\rho{=}2$ gate would skip at real savings; engines that select
matching orders from candidate-set sizes couple candidate reductions to
exploration (\S\ref{sec:exp:method}); and engines scanning materialized
candidate lists pay per-candidate costs that pruning does reduce
(\S\ref{sec:exp:static}). We did not encounter the first configuration
in any natural workload; we delimit rather than universalize.

The regularity delimits, rather than condemns, the machinery: consumers
that are \emph{not} adjacency-guided---periodic full re-enumeration,
batch candidate-space construction (the static setting, where
PILOS~\cite{pilos} demonstrably gains), admission control for
approximate matching---retain the full benefit of \selspec-maintained
spectra and \gate{}s, with Theorems~\ref{thm:complete}
and~\ref{thm:gate} intact.

\section{Experiments}
\label{sec:experiments}

Our evaluation has two layers. \emph{Probes} (Python/LAPACK) isolate
each mechanism's intrinsic behavior; \emph{framework experiments}
(C++17, single-threaded) integrate \selspec\ and the gates into the
decoupled CSM benchmark of Gou et al.~\cite{csmexperiments} as a drop-in
index (\texttt{index\_type~6}), against the strictest control: the
identical pipeline with the identical NLF index minus the spectral
layer (NLFI), with the enumeration engine fixed (NewSP-style). This
isolates the tested signal and respects the implementation-sensitivity
lesson of~\cite{deltacompile}. Data: synthetic Erd\H{o}s--R\'enyi
(avg.\ degree 3) and Watts--Strogatz families; real graphs ca-GrQc
(4.2K vertices, collaboration) and email-Enron (33.7K vertices,
communication) and com-Amazon (334.9K)~\cite{snap}, with uniform labels
$L\in\{4,8\}$ and a $90{:}10$ initial/insert-stream split; update
streams are uniformly shuffled (these graphs carry no timestamps), the
standard protocol for untimestamped data~\cite{sun2022indepth}. Queries are
dense BFS-ball samples following standard CSM
practice~\cite{sun2022indepth}: a \emph{main} workload of sizes
$16$--$32$ (five per size) and a \emph{thickening} workload of sizes
$8$/$12$ (ten per size) whose easier instances mostly solve within
budget. Eigensolves use LAPACK \texttt{dsyev}; $h{=}2$, $r\in\{4,8\}$,
$S_{\max}{=}64$ unless noted. All experiments run single-threaded on an
AMD Ryzen Threadripper PRO 3945WX (12 cores, 64\,GB RAM) under
Ubuntu~24.04, GCC~13.3 at \texttt{-O3}. The per-query time budget is
$60$s throughout, covering stream processing only and checked between
updates---a query is stopped after the update that crosses the line,
so index maintenance within that update (and per-query graph loading,
excluded from the budget) can push wall times past $60$s; key results
are re-verified at $5\times$ and $10\times$ budgets. Table~\ref{tab:datasets} lists the graphs; probe experiments
(\S\ref{sec:exp:probes}) use $n{=}3{,}000$--$8{,}000$ synthetics and
ca-GrQc. All counters are deterministic. Wall times are single-run; a
three-repeat check on the thickening workload bounds run-to-run spread
at $1.1\%$ (control) and $1.4\%$ (spectral) of the median, an order of
magnitude below every time difference we interpret. We define the
central metric here: an enumeration \emph{intermediate} is one binding
of a data vertex to a non-anchor query vertex followed by a descent of
the backtracking search (the counter incremented at each recursive
extension; engines count at different sites, which is immaterial since
all comparisons are within-engine); raw totals are reported on solved
queries.

\begin{table}
\caption{Graphs used in framework experiments.}
\label{tab:datasets}
\small
\begin{tabular}{lrrrl}
\toprule
Graph & $|V|$ & $|E|$ & avg deg & type \\
\midrule
ca-GrQc & 4.2K & 13.4K & 6.5 & collaboration \\
email-Enron & 33.7K & 180.8K & 10.7 & communication \\
com-Amazon & 334.9K & 925.9K & 5.5 & co-purchase \\
com-Youtube & 1.13M & 2.99M & 5.3 & social \\
ER / WS (probes) & 3K & 4.5--6K & 3--4 & synthetic \\
\bottomrule
\end{tabular}
\end{table}

\subsection{Probe Results}
\label{sec:exp:probes}

\begin{figure*}[t]
\centering
\includegraphics[width=0.98\textwidth]{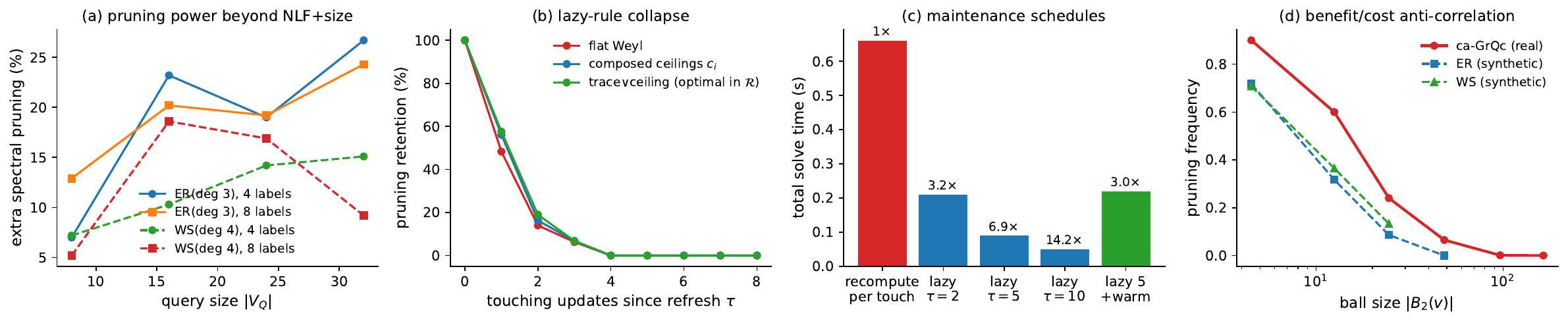}
\caption{Feasibility probes. (a)~Spectral pruning power beyond NLF+size
grows with query size on sparse graphs. (b)~Lazy-rule pruning retention
collapses within $\tau\le 4$ touching updates for flat Weyl, the
composed ceilings $c_i$, and the trace$\vee$ceiling rule, which is
optimal in the perturbation relaxation (Theorem~\ref{thm:tight}).
(c)~Maintenance cost scales as $(\tau{+}1)\times$; warm-started
iterative solvers lose to direct solves at deployment-scale
neighborhoods. (d)~The benefit/cost anti-correlation on two synthetic
families and a real collaboration network.}
\label{fig:probe}
\end{figure*}

Figure~\ref{fig:probe} summarizes.
\textbf{(a) Pruning power exists and grows with query size.} On sparse
graphs with dense queries, exact interlacing removes $7$--$27\%$ of
candidates beyond NLF and size screening (q32, ER: $26.7\%$),
replicating the magnitude PILOS reports statically; on tree-like
random-walk queries it removes ${\approx}0\%$---the applicability
domain is sharp.
\textbf{(b) Lazy bounds collapse} (\S\ref{sec:collapse}): retention
$54$/$56$/$58\%$ at $\tau{=}1$, ${\le}19\%$ at $\tau{=}2$, $0\%$ at
$\tau{=}4$ for flat Weyl, the composed ceilings $c_i$, \emph{and} the
optimal trace$\vee$ceiling rule; with naive per-update (rather than
per-edge) accounting we additionally observed genuine safety
violations, confirming the accounting requirement of
\S\ref{sec:lazyfamilies}.
\textbf{(c) Maintenance.} Lazy scheduling saves only $(\tau{+}1)\times$,
moot given (b); selective exact recomputation costs $52$--$225\mu s$
per solve in the deployment range, and with the exact trigger rule
(both endpoints inside the ball; $4.55$ true triggers per update vs.\
$20.3$ under the naive superset rule) the end-to-end kernel cost is
$313\mu s$ per update on ca-GrQc.
\textbf{(d) Anti-correlation law} on three graph families; on ca-GrQc,
$S_{\max}{=}32$ covers $63.6\%$ of vertices and retains $96.1\%$ of
pruning, $S_{\max}{=}64$ retains $100\%$; the top two eigenvalues fire
$88$--$100\%$ of prunes.

\subsection{The Boundary Regularity In-Framework}
\label{sec:exp:law}

\begin{table}
\caption{Replication matrix for intermediate invariance. Every cell is
a configuration in which all co-solved queries had exactly equal match
totals and intermediates; ``solved'' counts queries solved by both
sides.}
\label{tab:matrix}
\small
\begin{tabular}{llllr}
\toprule
Test & Engine & Graph & Stream & solved \\
\midrule
vertex (q16--32) & NewSP & GrQc & ins & 3 \\
vertex+gates (q8/12) & NewSP & GrQc & ins & 18 \\
vertex+gates (q8/12) & SymBi & GrQc & ins & 20 \\
vertex+gates & NewSP & GrQc & 70:30 & 2 \\
vertex+gates (Zipf) & NewSP & GrQc & ins & 5 \\
vertex+gates (sparse) & NewSP & GrQc & ins & 9 \\
vertex+gates & NewSP & Enron & ins & 13 \\
vertex+gates & NewSP & Amazon & ins & 8 \\
gate-only (lazy) & NewSP & Amazon & ins & 10 \\
\bottomrule
\end{tabular}
\par\smallskip
\footnotesize Gate-bearing rows are equal up to skipped doomed
bindings: $0$--$1$ per query where raw counters were captured on
Enron, $1$--$8$ on the com-Amazon gate runs ($\le 2.2\%$ of that
query's bindings, \S\ref{sec:exp:amazon}); vertex-only rows are exact.
\end{table}

For context beyond the NLFI control, we also ran the framework's full
index family (NI, DCG, DCS, CaLiG) on the same workload with the
expected relative behavior.\footnote{CaLiG's kernel-and-shell search is
sensitive to this dense low-selectivity workload (one query spent
$1{,}497$s inside a single uninterruptible update enumeration); we
report its seven queries completing within a 90-minute envelope. None
of this affects the controlled comparison, which varies only the
spectral layer.} We do not compare against DIVINE~\cite{divine}
experimentally: its dominance-embedding synopses are coupled to its own
engine rather than exposed behind the decoupled framework's index
interface, so any end-to-end gap would conflate
signal with implementation---the confound the framework exists to
remove~\cite{deltacompile}. Our lens applies to its signal nonetheless
(\S\ref{sec:related}).

Table~\ref{tab:invariance} is the paper's central table. Five
configurations---vertex-level interlacing at three query sizes, the
Laplacian anchored gate, and the triangle-extended gate---all prune
substantially by their own meter, never violate completeness (match
totals identical to NLFI on all $20$ solved queries, ranging from $1$
to $2.8\times 10^{8}$ matches), and leave the per-update
intermediate-result count \emph{exactly equal}. The gates' fired
updates are exactly those whose enumerations were already free: on the
hardest solved query (${\sim}23$s), zero of eight processed updates
fired any gate. Maintenance overhead, not enumeration, accounts for the
entire residual time difference ($+7$--$31\%$ on solved queries), and
an overhead-reduction pass (reverse-index compaction, change-gated
resync; memory $-40\%$) left time unchanged---consistent with the law:
even at zero overhead there is no enumeration time to win.

\begin{table}
\caption{Intermediate invariance by test configuration (ca-GrQc,
$L{=}8$ unless noted; cross-engine/graph/stream replications in
Table~\ref{tab:matrix}). ``Prune'' is the configuration's own meter:
candidate-vertex reduction (vertex level) or gate fire rate among
processed updates (update level). Intermediates equal the NLFI control
in every configuration; match counts always agree.}
\label{tab:invariance}
\small
\setlength{\tabcolsep}{3pt}
\footnotesize
\begin{tabular}{@{}lllcc@{}}
\toprule
Test & Granul. & Prune & Interm. & Compl. \\
\midrule
Interlacing, q16 & vertex & $-9..{-}36\%$ & equal & \checkmark\\
Interlacing, q24 & vertex & $-23..{-}51\%$ & equal & \checkmark\\
Interlacing, q32 & vertex & $-10..{-}42\%$ & equal & \checkmark\\
Laplacian \gate & update & $5..45\%$ & equal & \checkmark\\
\;+ triangles & update & $7..47\%$ & equal & \checkmark\\
\bottomrule
\end{tabular}
\par\smallskip
\footnotesize ``Equal'' is exact for vertex-level rows; for gate rows,
raw totals can differ by the gate's skipped doomed bindings---$0$--$8$
per query on natural workloads ($\le 2.2\%$, typically $\ll 0.1\%$),
first-level bindings of matchless updates and thus precisely the
``work discarded within first expansion levels'' of the regularity.
Prune rates are measured on the \emph{main} workload
(dense queries, $|V_Q|\in\{16,24,32\}$, five per size). Invariance is
asserted on solved queries only (truncated runs progress differently):
the main workload solves two ($|V_Q|{=}24$) at the default budget and a
third ($|V_Q|{=}16$) at $10\times$---the latter with raw intermediates
exactly equal at $1.56\times 10^{9}$ bindings and $2.13\times 10^{9}$
matches, the strongest single datum in the paper---and a
\emph{thickening} workload of twenty easier dense queries
($|V_Q|\in\{8,12\}$, ten per size) solves eighteen more. Invariance is
further stable under $5\times$/$10\times$ budgets, across an
insert/delete-mixed stream (\S\ref{sec:exp:mixed}), under a second
enumeration engine on all twenty thickening queries
(\S\ref{sec:exp:engine2}), and on two further graphs
(\S\ref{sec:exp:enron}--\ref{sec:exp:amazon}).
\end{table}

\subsection{Cross-Graph Replication}
\label{sec:exp:enron}

On email-Enron ($33.7$K vertices, avg.\ degree $10.7$; $18{,}081$-insert
stream) the regularity replicates under its strongest single test: the
main-workload query solved within budget processes the \emph{entire}
update stream, the triangle-extended gate fires on $15.7\%$ of its
updates ($2{,}848$ skipped enumerations, $1{,}247$ triangle disproofs,
zero wrong skips---$121{,}157$ matches, identical to control), with
intermediates-per-update exactly equal ($611.857$). The thickening
workload adds twelve more solved queries on this graph: all twelve
match totals agree (up to $2.86\times 10^{9}$ matches on one query);
under raw counters, intermediates are exactly equal on ten (e.g.\
$6{,}801{,}976$ on both sides of the heaviest) and differ by a
\emph{single} gate-skipped doomed binding on two ($95{,}915$ vs.\
$95{,}916$, equal denominators)---the continuum of
\S\ref{sec:exp:amazon} at its natural floor. The
$+3.5$s residual on this query equals the measured per-update
maintenance-plus-gate cost times the stream length, confirming once
more that the entire time difference is overhead, none of it
enumeration. Candidate reductions are smaller here ($1$--$7\%$) than on
the sparse collaboration graph: with average degree $10.7$, balls are
larger, fewer vertices are deployable, and fewer are spectrally
weak---the anti-correlation law (\S\ref{sec:selspec}) predicting its
own diminishing domain.

\subsection{At Scale: the Overhead Verdict}
\label{sec:exp:amazon}

com-Amazon ($334.9$K vertices, avg.\ degree $5.5$; $92{,}587$-insert
stream) completes the picture, a further invariance replication with a
twist. Sparsity makes this the spectral test's best regime: vertex-level
pruning fires up to $415$K times per query and removes up to $51\%$ of
candidates. It is also the regime where nearly every vertex is
deployable, so every update refreshes ${\sim}5$ cached spectra: the
control completes all ten queries in $46$s total while the spectral
configuration needs over ten minutes---match counts identical on every
query both sides finish ($10{,}080=10{,}080$ on the largest), and
intermediates equal up to the gate's skipped doomed bindings, $1$--$8$
per query under raw counters (both gate-bearing configurations skip the
same updates; e.g.\ $359$ vs $351$ on one light query,
$19{,}670{,}525$ vs $19{,}670{,}524$ on the heaviest). Pruning at its
most active, enumeration gains at a few first-level bindings: this is
the boundary regularity priced out. Two design points separate the
mechanism from the engineering. The \emph{eager} configuration
(materialized spectral candidate maps require touch-time recompute and
bidirectional resync for completeness) pays ${\sim}0.5$ms per update
and accounts for the ten minutes. A \emph{lazy gate-only}
configuration---drop the vertex layer entirely (for edge-view-driven
engines its read-path tests never reach exploration anyway,
\S\ref{sec:weakcand}) and keep the self-contained \gate---completes
all ten queries in $1$m$55$s with identical matches; raw intermediates
differ only by the gate's skipped doomed bindings, $1$--$8$ per query
($19{,}670{,}524$ vs.\ $19{,}670{,}525$ on the largest)---the
continuum's natural end, against the constructed workload's
$-99.9\%$ (\S\ref{sec:exp:adversarial}).
So ``not yet engineered well'' is partly true of the eager design: a
$5\times$ overhead reduction exists inside the design space. What no
engineering can change is the numerator: with intermediates provably
untouched, the best achievable online outcome is the control's
runtime plus a nonzero gate cost---still a strict loss on every natural
workload we measured (\S\ref{sec:exp:adversarial} shows the constructed
exception).

At com-Youtube scale ($1.13$M vertices, $2.99$M edges,
$298{,}762$-insert stream) the design split repeats. The eager
configuration is feasible (init $+2$--$14$s, index $25$--$39$MB) but
solves \emph{neither} of the control's two in-budget queries ($32.4$s,
$11.9$s)---its maintenance overhead alone converts both into timeouts.
The lazy gate-only configuration recovers both, with identical match
totals ($2.08\times 10^{8}$ and $1.59\times 10^{8}$) at $50.9$s and
$38.6$s---the residual being pure gate cost over a $299$K-update
stream---and its total wall time over the twenty-query thickening
workload sits within $4\%$ of the control's, the other eighteen queries
timing out on both sides. Million-scale, then, confirms both halves:
overhead engineering matters (eager loses solved queries; lazy keeps
them), and no engineering recovers a win (the lazy variant still pays
the gate for intermediates that never change).

\subsection{The Positive Boundary: Where Aggregates Do Pay}
\label{sec:exp:static}

We rebuilt the static pipeline---candidate generation, CECI-style
candidate-space (CS) construction, full backtracking enumeration---and
compared NLF against NLF+spectral candidates on ER and ca-GrQc with
dense q16 queries. CS construction cost, metered
implementation-independently as adjacency-scan operations, drops in
near proportion to candidates ($-2$ to $-43\%$, tracking the $-4$ to
$-38\%$ candidate reductions); match outputs agree exactly on every
uncapped query. Enumeration \emph{nodes}, however, remain identical
even here: our enumerator draws its pools from CS edges and is thus
adjacency-guided too. This sharpens the boundary regularity into an
architectural statement:

\begin{quote}
\emph{Aggregate pruning accelerates exactly the cost terms that scale
with candidate-set size (candidate-space construction, candidate-list
intersection), and never adjacency-guided exploration, in static and
dynamic settings alike.}
\end{quote}

Static matchers pay a large construction term per query, which is where
PILOS's measured gains live; delta enumeration pays no construction
term at all (its index is maintained incrementally), leaving aggregates
nothing to accelerate. The practical guidance follows: spend aggregate
tests where candidate sets are \emph{materialized or scanned}
(batch/static evaluation, index (re)builds, admission control), not
where exploration is pointer-chasing through adjacency.

\subsection{A Constructed Exception: the Instrument Can Detect
Non-Invariance}
\label{sec:exp:adversarial}

The regularity's mechanism (\S\ref{sec:law}) names its own escape
hatch---a region abundant at radius~1 yet deficient at radius~2---and
a fair objection is that an instrument that has only ever reported
equality might be unable to report anything else. We therefore
constructed the escape hatch. Query: a near-clique core ($K_8$ minus
one edge) with a two-hop tail, $10$ vertices, maximum degree $7$.
Data: $50$ decoy regions (isolated $K_8$, each missing one edge whose
arrival is the update) matching the query's degree profile
exactly---NLF admits every decoy vertex for every query vertex, and the
control must discover failure by exhausting injectivity ($10$ query
vertices, $8$ decoy vertices) through deep partial assignments; plus
one true region with the tail, completed by the final update.
Result: the control performs $7{,}543{,}814$ intermediate extensions in
$8.98$s; \gate's vertex-count screen ($|\ball{2}{\{a,b\}}|=8 <
|Q_2|=9$) disproves every mapping on all $50$ decoy updates and passes
the true one, yielding $7{,}942$ intermediates ($-99.9\%$) in
$0.012$s ($748\times$), with identical match totals ($720$). The
instrument detects non-invariance decisively when exploration crosses
an aggregate-visible boundary; on natural workloads it never does---
that contrast \emph{is} the finding.

\subsection{Robustness: Label Skew, Sparse Queries, Borderline Budgets}
\label{sec:exp:robust}

Under Zipf-distributed labels ($s{=}1.5$), five of ten queries solve
and all five agree exactly in totals and intermediates. (All label
assignments in this paper are synthetic, per standard practice for
unlabeled SNAP graphs; replaying natively labeled benchmark streams
such as LSBench/Netflow~\cite{sun2022indepth,csmexperiments} is the
immediate external-validity extension.) Under sparse
random-walk queries---the regime where the probe predicts near-zero
spectral pruning---candidate reductions indeed shrink to
$0.5$--$4\%$, nine of ten queries agree exactly, and the domain
boundary closes from the inside: where the tests cannot fire, nothing
changes by construction. One borderline case is instructive: a query
the control finishes at $59.5$s (of a $60$s budget) is pushed past the
line by the spectral configuration's maintenance overhead---it finds
the identical $8.45\times 10^{8}$ matches but lands at $60.8$s.
Overhead does not merely fail to help; near budget boundaries it can
flip outcomes.

\subsection{Deletion-Mixed Streams}
\label{sec:exp:mixed}

Deletions are the easy direction for safety---they only lower true
spectra, so stale entries over-admit and never over-prune
(\S\ref{sec:selspecindex})---but they exercise the bidirectional resync
and eviction machinery. We extended the framework's driver to
$70{:}30$ insert/delete streams (a patch we release with our artifacts)
and re-ran the ca-GrQc workload: all configurations complete, solved
queries produce identical outputs and identical intermediates across
control and spectral runs, and no safety anomaly appears in
$5{,}750$ exercised deletions, consistent with the over-admit argument.
In these runs \gate\ is applied to insertions only; the pre-deletion
form licensed by Theorem~\ref{thm:gate}'s symmetric statement is left
unexercised.

\subsection{A Second Engine, and the Pruned-Visited Intersection}
\label{sec:exp:engine2}

Two checks address the most natural objections. First, the framework's
decoupling lets us swap the enumeration engine: under the SymBi-style
engine (a different search process with its own intermediate counter),
all $20$ thickening queries solve on both sides, all $20$ match totals
agree---up to $1.5\times 10^{9}$ matches on a single query---and all
$20$ per-engine intermediate counts are again pairwise identical.
The regularity is not an artifact of one search process.
Second, we measured the claim behind the regularity \emph{directly}
rather than through counter equality: instrumenting the engine to
record every (query-vertex, data-vertex) binding it performs and the
index to record every binding the spectral test rejects, accumulated
over the whole stream. The overlap is $0$--$7$ bindings per query
(${\approx}0.5\%$ of visited bindings; exactly $0$ on $11$ of $20$
queries). The residue is not pruned-yet-visited leakage: at any moment
enumeration draws only from the candidate map, which excludes
currently-rejected bindings by construction; the overlapping bindings
are those whose status legitimately \emph{flipped} during the
stream---rejected at one point, re-admitted by the bidirectional resync
when an insertion raised the candidate's spectrum
(\S\ref{sec:selspecindex}), and visited only while admitted. The
accumulated overlap thus measures the resync mechanism at work, and its
sparseness shows how rarely the spectral frontier moves across
enumeration-relevant bindings.

\subsection{Methodology: the Intermediate-Invariance Test}
\label{sec:exp:method}

Candidate counts misled in both directions in our study: they
\emph{overstate} usefulness at the vertex level
(Table~\ref{tab:invariance}) and can \emph{understate} it for
edge-view-rich indexes~\cite{csmexperiments}. We propose reporting, for
any new CSM filter, the intermediate count with the filter on and off
under a fixed engine \emph{and fixed matching order}: a filter that
does not move this number cannot move enumeration time, no matter what
it does to candidates. Two qualifications. (i)~The test is
\emph{engine- and order-relative}: our primary engines derive matching
orders from the query graph alone (verified in the order-generation
code), so candidate reductions cannot feed back into exploration. We
probed the feedback channel directly by switching to a
candidate-cardinality-driven order (greedy smallest-candidate-first
from the anchor) and re-running the thickening workload: match totals
stay identical, but intermediates now move \emph{in both directions}---
one query inflates from $5.4\times 10^{7}$ to $1.2\times 10^{8}$
bindings ($17.8$s${\to}41.8$s) while another drops $32\%$---because the
spectral layer's candidate reductions reorder exploration
unpredictably. Under order feedback the invariance test certifies the
filter's direct effect only; whole-system claims additionally need the
order held fixed, and practitioners adding any filter to a
candidate-driven-order system should expect bidirectional, workload-
dependent swings unrelated to the filter's pruning quality.
(ii)~For a deterministic engine the equality itself is expected
whenever pruned candidates are never visited; the informative
measurements are the fire rates alongside the invariance, and the
direct intersection of \S\ref{sec:exp:engine2}.

\section{Related Work}
\label{sec:related}

\paragraph{CSM systems.} From early active-graph systems
\cite{graphflow}, four generations of index design---TurboFlux's
data-centric spanning tree~\cite{turboflux}, SymBi's bidirectional
dynamic candidate space~\cite{symbi}, RapidFlow~\cite{rapidflow},
CaLiG's candidate-lighting graph with kernel-and-shell
search~\cite{calig}, and NewSP's compatible-set reuse~\cite{newsp}---
each improved over its predecessor by one to three orders of
magnitude, all filtering with labels, degrees, and neighbor-label
multisets, refined by consistency propagation over candidate adjacency
\cite{csmsurvey}. DIVINE~\cite{divine} is the first to add a
non-structural signal (vertex-dominance embeddings over degree-grouped
star synopses, maintained incrementally). Its synopses are built from
degree-grouped star substructures---degree-sequence-adjacent
information in the sense of our GMB analysis---and our methodology
supplies the test any such signal should face: report intermediates
with the filter on and off under a fixed engine and order; unless the
prunable set intersects the cost-bearing set of adjacency-guided
enumeration, candidate reductions will not become time.

\paragraph{Experimental studies.} Sun et al.~\cite{sun2022indepth}
unified six CSM algorithms under incremental view maintenance and
identified scenario-dependent index choice; the delta-compilation study
of Lee et al.~\cite{deltacompile} showed an optimized implementation of
an old method beats CaLiG by up to $48.6\times$, making
implementation-controlled comparison mandatory, echoing the
common-codebase discipline of earlier static-matching
studies~\cite{subgraphcompare2012}---our evaluation
therefore swaps only the index layer inside the decoupled framework of
Gou et al.~\cite{csmexperiments}, whose caveat that candidate counts can
\emph{understate} filtering power we extend to the converse direction
with the intermediate-invariance test. The SIGMOD'24 survey of static
matching~\cite{sigmod2024survey} evaluates 534 combinations and finds
enumeration dominant; our law refines where that dominance does and does
not leave room for filtering.

\paragraph{Spectral and static matching.} PILOS~\cite{pilos}
introduced interlacing on neighborhood Laplacians for static matching,
building on candidate-space frameworks~\cite{ceci,inmemory2020,gup} and a
static-matching lineage scaling exploration to billion-node
graphs~\cite{billionnode2012}; its index is offline and static. We show its natural dynamic extensions are either
infeasible (lazy bounds) or ineffective for delta enumeration (the
boundary regularity), while its static gains are consistent with our
positive-boundary analysis. The Grone--Merris
conjecture~\cite{gronemerris}, proved by Bai~\cite{bai2011}, underpins
our explanation that Laplacian interlacing is largely a
degree-sequence statistic on matching workloads. Dynamic spectral
sparsifiers~\cite{dynsparsifier} maintain global spectral
approximations under updates, and incremental eigenpair tracking
follows leading eigenpairs of a single evolving
matrix~\cite{eigtrack}; maintaining exact local \emph{neighborhood}
spectra for all (deployable) vertices of a dynamic graph, as \selspec\
does, appears new to our knowledge.

\section{Conclusion}
\label{sec:conclusion}

We set out to bring spectral filtering to continuous subgraph matching
and instead mapped its boundary. Lazy maintenance is provably hopeless
in the useful regime; exact maintenance is cheap if deployed where an
anti-correlation law directs it; and the resulting pruning---however
safe, however large in candidate counts---cannot touch adjacency-guided
enumeration, a fact replicated across the eight configurations of Table 3 and
confirmed in our own static rebuild, where the same tests cut
candidate-space construction work by up to $43\%$ while leaving
exploration untouched. The regularity's final form is architectural: aggregate
invariants accelerate what scales with candidate sets, never what
chases adjacency. The positive residue is a reusable primitive (exact
dynamic local spectra), a safe update-level gating framework whose
yield lands wherever candidate sets are materialized, and a
methodological test (intermediate invariance) that we believe should
accompany every future claim that a new filter accelerates CSM.

\bibliographystyle{ACM-Reference-Format}
\bibliography{refs}

\end{document}